\documentclass[a4paper,10pt]{article}

\title{Network Creation Games with Local Information and Edge Swaps} 

\author{Shotaro Yoshimura\thanks{Graduate School of Information Science and Electrical Engineering, Kyushu University, Japan. E-mail: \texttt{yoshimura@tcs.inf.kyushu-u.ac.jp}} \and 
Yukiko Yamauchi\thanks{Corresponding author. Faculty of Information Science and Electrical Engineering, Kyushu University, Japan. E-mail: \texttt{yamauchi@inf.kyushu-u.ac.jp}}}

\newcommand{\Order}{\mathrm{O}}

\newcommand{\argmax}{\mathop{\rm arg~max}\limits}

\usepackage{amsmath,amsfonts} 
\usepackage{version}
\usepackage{algorithmic, algorithm}
\usepackage{theorem}
\usepackage[pdftex]{graphicx}
\usepackage{latexsym}
\usepackage{subfig}
\usepackage{bm}

\theoremstyle{plain}
\newtheorem{theorem}{Theorem}
\newtheorem{lemma}[theorem]{Lemma}
\newtheorem{corollary}[theorem]{Corollary}

\newenvironment{proof}{{\bf Proof. } }

\newtheorem{example}{Example}
\newcommand{\PoA}{\mathrm{PoA}}

\begin{document}
\date{}
\maketitle

\begin{abstract}
In the \emph{swap game (SG)} selfish players, 
each of which is associated to a vertex, 
form a graph by edge swaps, i.e., 
a player changes its strategy by simultaneously 
removing an adjacent edge and forming a new edge (Alon et al., 2013). 
The cost of a player considers the average distance to all other players or 
the maximum distance to other players. 
Any SG by $n$ players starting from a tree converges to an equilibrium with 
a constant Price of Anarchy (PoA) within 
$\Order(n^3)$ edge swaps (Lenzner, 2011). 
We focus on SGs where each player knows the subgraph induced by 
players within distance $k$. 
Therefore, each player cannot compute its cost nor a best response. 
We first consider \emph{pessimistic} players who consider 
the worst-case global graph. 
We show that any SG starting from a tree 
(i) always converges to an equilibrium within $\Order(n^3)$ edge swaps 
irrespective of the value of $k$,
(ii) the PoA is $\Theta(n)$ for $k=1,2,3$, and 
(iii) the PoA is constant for $k \geq 4$. 
We then introduce \emph{weakly pessimistic} players and 
\emph{optimistic} players and show that 
these less pessimistic players achieve constant PoA 
for $k \leq 3$ at the cost of best response cycles.  

\end{abstract}

\section{Introduction}

Static and dynamic properties of networks 
not controlled by any centralized authority 
attracts much attention in last two decades 
as self-organizing large-scale networks play 
a critical role in a variety of information systems, 
for example, the Internet, 
Peer-to-Peer networks, 
ad-hoc networks, wireless sensor networks, 
social networks, viral networks, and so on. 
In these networks, participants selfishly and rationally change 
a part of the network structure to 
minimize their cost and maximize their gain. 
Controlling such networks is essentially 
impossible and many theoretical and empirical studies have 
been conducted; 
stochastic network construction models 
such as 
the Barab\'{a}si–Albert model 
were proposed, and 
key structural properties such as 
the small world networks~\cite{WS98} and 
the scale-free networks~\cite{BB03} 
have been discovered. 
Stochastic communication models such as 
the voting models~\cite{DGMSS11,NIY99}, 
the random phone call model~\cite{KSSV00}, and 
the rewiring model~\cite{DGLMSSSV12} were proposed and 
many phase transition phenomena have been reported. 
Many problems related to broadcasting, gossiping, and 
viral marketing were also proposed~\cite{BCEG10,DR01,KKT03}.

In this paper, we take a game-theoretic approach to 
analyze dynamics and efficiency of the network structure 
resulting from local reconstruction by selfish agents. 
The \emph{network creation game (NCG)} considers 
$n$ players forming a network~\cite{FLMPS03}.  
Each player is associated with a vertex of the network, 
can construct a communication edge 
connecting itself to another player at the cost of $\alpha$, 
and can remove an adjacent edge for free. 
The cost of a player is the sum of the \emph{construction cost} 
for edges and the \emph{communication cost}, which is the 
sum of distances to all other players in the current network, 
i.e., the average distance to other players. 
Each player selfishly changes its strategy to minimize its cost and 
the social cost of a network is the sum of all players' costs. 
The \emph{Price of Anarchy} (PoA) of NCG is constant for almost all values 
of $\alpha$~\cite{AEEMR14,DHMZ12,MMM13,MS13}, 
yet the PoA is not known for some values of $\alpha$. 
However, computing the best response in NCG is 
NP-hard~\cite{FLMPS03}, and this fact makes the NCG unrealistic 
in large-scale networks. 
The NCG with another type of communication cost is proposed in 
\cite{DHMZ12}, where the cost of a player is the maximum distance 
to other players. 
We call this game the 
\emph{Max Network Creation Game (MAX-NCG)} and 
the original NCG the \emph{Sum Network Creation Game (SUM-NCG)}. 
However, the SUM-NCG and the MAX-NCG ignores 
one of the most critical limitations in 
large-scale networks; 
each player cannot obtain ``global'' information. 
This type of locality is a fundamental limitation 
in distributed computing~\cite{G01}, 
although players can neither compute its cost nor 
the best response without global information. 

In this paper, we focus on games in 
such a distributed environment where 
each player cannot obtain the current strategy of all players 
nor have enough local memory to store the global information. 
Rather, players can access only local information. 
The NCG by players with local information is first proposed in 
\cite{BGLP16}. 
Each player can observe a subgraph of the current 
graph induced by the players within distance $k$. 
We call this information the \emph{$k$-local information}. 
The players are \emph{pessimistic} in the sense that 
they consider the worst-case global graph when they 
examine a new strategy. 
Computing the best response for MAX-NCG is still NP-hard because 
$k$-local information may contain the entire network. 
For small $k$, more specifically, 
for $1 \leq k \leq \alpha+1$, 
$\PoA=\Omega(\frac{n}{1+\alpha})$ for MAX-NCG 
and 
for $k \leq c\sqrt[3]{\alpha}$ 
$\PoA=\Omega(n/k)$ for SUM-NCG. 
These results contrast global information with local information. 
The SUM-NCG and MAX-NCG by players with global trace-route 
based information is proposed, 
yet $\PoA=\Theta(n)$ for some values of $\alpha$~\cite{BGLP14}. 
The NCG for more powerful players 
with $k$-local information is considered in \cite{CL15}, 
where the players can probe the cost of a new strategy. 
Computing the best response is NP-hard 
for any $k \geq 1$ while there exists tree equilibrium 
that achieves $\PoA=\Order(\log n)$ and 
$\PoA=\Omega(\frac{\log n}{k})$ 
for $2 \leq k \leq \log n$ 
and $\PoA=\Theta(n)$ for $k=1$. 
For non-tree networks, depending on the values of $\alpha$ and 
$k$, we have $\PoA=\Order(n)$. 

The \emph{swap game} (SG) restricts strategy changes to
\emph{edge swaps}, i.e., simultaneously 
removing an edge and creating a new edge~\cite{ADHL13}. 
Thus, any strategy change does not change the number of edges 
in the network and the best response can be computed in 
polynomial time. 
Additionally, when we restrict initial networks to trees, 
a star achieves the minimum social cost. 
Above mentioned cost functions were adopted 
and these SGs are called the SUM-SG and the MAX-SG, respectively. 
The aim of SG is to omit parameter $\alpha$ from NCG 
with keeping the essence of NCG. 
The authors showed that the diameter of a tree 
equilibrium is two for the SUM-SG and at most three for the MAX-SG, 
while there exists an equilibrium with 
a large diameter in general networks. 
Thus, PoA of a tree equilibrium is always constant. 
Moreover, any SUM-SG and MAX-SG starting from a tree 
converges to an equilibrium within $O(n^3)$ edge swaps 
while they admit best response cycles 
starting from a general graph~\cite{KL13,L11}. 
Consequently, local search at players with global information 
achieves efficient network construction for initial tree networks. 
The SUM-SG and MAX-SG with ``powerful'' players 
with $k$-local information is investigated in \cite{CL15}. 
For $k \geq 2$, 
the SUM-SG and MAX-SG starting from general networks 
admits best response cycles 
while convergence within $\Order(n^3)$ moves 
is guaranteed for tree networks. 
However, to the best of our knowledge, 
SG with $k$-local information has not been considered.

\subsection{Our results}
In this paper, we investigate the convergence property and 
PoA of SGs by players with local information.
First, we consider pessimistic players and 
demonstrate that starting from an initial tree, 
any SUM-SG and MAX-SG converge to an equilibrium within
$\Order(n^3)$ edge swaps in the same manner as 
\cite{L11}, i.e., 
we present a generalized ordinal potential function 
for the two games. 
We also show that convergence from a general network 
is not always guaranteed. 
Then, we present a clear phase transition phenomenon caused 
by the locality. 
\begin{itemize}
 \item When $k=1,2$, pessimistic players never perform any edge swap 
       in the SUM-SG and in MAX-SG. 
       Any network is an equilibrium of the two games, thus 
       $\PoA = \Theta(n)$. 
 \item When $k=3$, in the SUM-SG and MAX-SG, 
       there exists an equilibrium of diameter $\Theta(n)$, 
       thus $\PoA = \Theta(n)$. 
 \item When $k \geq 4$, in the SUM-SG and MAX-SG 
       the diameter of every equilibrium 
       is constant, thus PoA is constant. 
\end{itemize}
We then introduce \emph{weakly-pessimistic} players and 
\emph{optimistic} players to obtain a better PoA for $k \leq 3$. 
A weakly pessimistic player performs an edge swap 
even when its cost does not decrease. 
This relaxation results in a constant PoA of the MAX-SG 
when $k=3$ at the cost of best response cycles. 
An optimistic player assumes the best-case global graph 
for an edge swap and this optimism results in a 
constant PoA of the SUM-SG and MAX-SG for any value of $k$. 
Consequently, the combination of $k$-locality for $k \geq 4$
and pessimism enables distributed construction of efficient trees 
by selfish players.

\subsection{Related works} 

We briefly survey existing results of the NCG and SG 
for players with global information. 
Regarding the SUM-NCG, when $\alpha \le n^{1-\varepsilon}$ for
$\varepsilon \ge 1/\log{n}$ the PoA is 
$O(3^{1/\varepsilon})$~\cite{DHMZ12}. 
Thus, when $n$ is sufficiently large, 
the PoA is bound by a constant. 
When $\alpha > 4n+13$, 
the PoA is at most $3+2n/(2n+\alpha)$~\cite{BL18}. 
In addition, any constant upper bound of PoA 
for $n \le \alpha \le 4n+13$ is not known and 
the best upper bound is $O(2^{\sqrt{\lg{n}}})$~\cite{DHMZ12}. 
If every equilibrium is a tree, then $\PoA < 5$ and 
an interesting conjecture is that 
every equilibrium is a tree for sufficiently large 
$\alpha$~\cite{FLMPS03}.
Regarding the MAX-NCG, 
the PoA is $2^{\Order(\sqrt{\lg n })}$ and 
it is constant when 
$\alpha = O(n^{-1/2})$ or $\alpha > 129$~\cite{MS13}. 

Regarding the SUM-SG, there exists an equilibrium with 
diameter $2^{\Order(\sqrt{\lg n})}$ 
while the diameter of any equilibrium is at most two (thus, a star) 
if an initial graph is a tree~\cite{ADHL13}. 
Regarding the MAX-SG, there exists an equilibrium with 
diameter $\Theta(\sqrt{n})$ 
while the diameter of any equilibrium is at most three 
if an initial graph is a tree~\cite{ADHL13}.

\subsection{Organization} 
Preliminary section introduces the SGs and 
pessimistic players with local information. 
In the next section, we analyze the dynamics and PoA 
of SGs by pessimistic players. 
In the following section, 
we introduce less pessimistic players 
and present best response cycles 
and equilibria with small diameter.  
Finally, we conclude this paper with open problems.

\section{Preliminaries}

A \emph{swap game (SG)} by players with $k$-local information 
is denoted by $(G_0, k)$, 
where $G_0=(V, E_0)$ is an initial network and 
integer $k$ is the size of each player's ``visibility''. 
$G_0$ is a simple undirected connected graph, 
where $|V|=n$ and $|E_0|=m$. 
We say $u \in V$ is adjacent to $v \in V$ 
if edge $\{u,v\}$ is an element of $E$. 
Each player is associated to a vertex in $V$ and 
the \emph{strategy} of a player $u \in V$ 
is the set of its incident edges. 

Each player can change its strategy by an \emph{edge swap}, 
i.e., 
removing one incident edge and creating a new edge. 
Starting from $G_0$, a sequence of edge swaps generates 
a network evolution $G_0, G_1, G_2, \ldots$. 

Let $N_G(u)$ be the set of adjacent vertices of $u \in V$ in $G$ and 
$d_G(u,v)$ be the distance between $u,v \in V$ in $G$. 
When $G$ is not connected and $v$ is not reachable from $u$, 
$d_G(u,v) = \infty$. 
The cost of a player depends on the current graph $G$. 
We consider two different types of cost functions, 
$c_{\text{SUM},u}(G)$ and $c_{\text{MAX},u}(G)$ 
defined as follows: 
\begin{eqnarray*}
c_{\text{SUM},u}(G) &=& \sum_{v \in V} d_G(u,v) \\ 
c_{\text{MAX},u}(G) &=& \max_{v \in V} d_G(u,v). 
\end{eqnarray*}
When $G$ is not connected, $c_{\text{SUM},u}(G)=\infty$ and 
$c_{\text{MAX},u}(G)=\infty$. 
We call a swap game where each player $u$ uses  
$c_{\text{SUM},u}$ the \emph{sum swap game (SUM-SG)} 
and a swap game where each player $u$ uses 
$c_{\text{MAX},u}$ the \emph{max swap game (MAX-SG)}. 
When it is clear from the context, 
we omit the name of the game and use $c_u$. 

Each player $u$ can access local information determined by $G$.  
Let $V_{G,k}(u)$ denote the set of vertices within distance $k$ 
from $u$ in $G$ (thus, the $k$-neighborhood of $u$). 
Player $u$ can observe the subgraph of $G$ 
induced by $V_{G,k}(u)$ 
and we call this subgraph the \emph{view} of $u$. 
We say the information at $u$ is \emph{$k$-local} 
and we call its view the \emph{$k$-local information} of $u$. 
We assume that each player does not know any global information 
such as the values of $n$ and $m$. 

In a transition from $G_t$ to $G_{t+1}$, 
a single player performs an edge swap. 
Consider the case where a player $u$ performs an edge swap 
$(v,w) \in N_{G_t}(u) \times (V_{G_t,k}(u) \setminus 
(N_{G_t}(u) \cup \{u\} ))$ 
in $G_t$. 
We call $u$ the \emph{moving player} in $G_t$. 
The resulting graph is 
$G_{t+1} = (V, E\setminus\{\{u,v\}\}\cup\{\{u,w\}\})$. 
Note that the number of edges does not change in a SG. 

Due to local information, 
each player cannot compute its current cost nor 
the improvement by a strategy change. 
We first consider \emph{pessimistic} players that 
consider the worst-case improvement for each possible edge swap 
and select one that achieves positive improvement. 
A player $u$ is \emph{unhappy} if it has an edge swap 
that decreases its cost in the worst-case global graph. 
In other words, there exists at least one edge swap $(v,w)$ at 
$u$ that satisfies 
\begin{eqnarray*}
 \Delta_{u}(v,w) &=& \min_{H \in {\mathcal G}_u} (c_u(H) - c_u(H')) > 0, 
\end{eqnarray*}
where ${\mathcal G}_u$ is the set of simple undirected 
connected graphs consisting of finite number of vertices and 
compatible with $u$'s local view, and 
$H'$ is a graph obtained by the edge swap $(v,w)$ at $u$ in 
$H \in {\mathcal G}$. 
We assume that a moving player always performs 
an edge swap $(v,w)$ with $\Delta_u(v,w) > 0$. 
When every player $u$ is not unhappy with respect to 
$c_{\text{SUM},u}(G)$ in graph $G$, 
we call $G$ a \emph{sum-swap equilibrium}. 
When every player $u$ is not unhappy with respect to 
$c_{\text{MAX},u}(G)$ in graph $G$, 
we call $G$ a \emph{max-swap equilibrium}. 
When a graph is a sum-swap equilibrium and a max-swap equilibrium 
we simply call the graph \emph{swap equilibrium}. 

We define the \emph{social cost} $SC(G)$ of a graph $G$ as the sum of all players' costs, 
i.e., $SC(G) = \sum_{u \in V} c_u(G)$. 
Let ${\mathcal G}(n,m)$ be the set of simple undirected connected graphs 
of $n$ players and $m$ edges and 
$\overline{\mathcal G}_{\text{SUM}}(n,m,k)$ be the set of sum-swap 
equilibrium graphs 
of $n$ players with $k$-local information and $m$ edges. 
The \emph{Price of Anarchy (PoA)} of the SUM-SG 
is defined as follows: 
\begin{eqnarray*}
\PoA_{\text{SUM}}(n,m,k) = 
\frac{\max_{G \in \overline{\mathcal G}_{\text{SUM}}(n,m,k)} SC(G)}{\min_{G' \in {\mathcal G}(n,m)} SC(G')}. 
\end{eqnarray*}
In the same way, the PoA of the MAX-SG is defined for 
the set $\overline{\mathcal G}_{\text{MAX}}(n,m,k)$ of max-swap 
equilibrium graphs 
of $n$ players with $k$-local information and $m$ edges. 
The PoA of the SUM-SG (and MAX-SG) starting from a tree 
is denoted by $\PoA_{\text{SUM}}(n, n-1, k)$ 
($\PoA_{\text{MAX}}(n, n-1, k)$, respectively). 

A strategic game has the \emph{finite improvement property} (FIP) 
if every sequence of improving strategy changes is finite~\cite{MS96}. 
Thus, from any initial state, 
any sequence of finite improving strategy changes reaches 
an equilibrium. 
Monderer and Shapley showed that a strategic game has the FIP 
if and only if it has a \emph{generalized ordinal potential function}. 
Regarding a swap game, 
a function $\Phi: {\mathcal G}_{n,m} \to {\mathbb R}$ 
is a generalized ordinal potential function 
if we have the following property for every graph 
$G \in {\mathcal G}_{n,m}$, every unhappy player $u$, 
and every edge swap $(v,w)$ that makes $u$ unhappy, 
\begin{eqnarray*}
c_u(G)-c_u(G') > 0 &\Rightarrow & \Phi(G)-\Phi(G') > 0, 
\end{eqnarray*}
where $G'$ is a graph obtained by the edge swap $(v,w)$ at $u$. 
That is, any transition in the SUM-SG and MAX-SG 
satisfies the above property for the moving player. 

The \emph{best response} of a player $u$ in $G_t$ is 
an edge swap $(v,w)$ that maximizes $\Delta_{u}(v,w)$. 
We call an evolution $G_0, G_1, G_2, \ldots, G_i(=G_0)$ 
a \emph{best response cycle} when 
each moving player in $G_t$ performs a best response 
for $t=0,1,2,\ldots, i-1$. 

We further introduce some notations for graph $G=(V,E)$. 
For a set of vertices $V' \subseteq V$ 
the graph obtained by removing vertices in $V'$ and 
their incident edges 
is denoted by $G \setminus V'$. 
Additionally, for a set of edges $E' \subseteq E$ 
the graph obtained by removing edges in $E'$ is 
denoted by $G \setminus E'$. 
The vertex set and the edge set of a graph $G'$ 
is denoted by $V(G')$ and $E(G')$, respectively.

\section{Convergence properties for pessimistic players }

In this section, we investigate the dynamics of 
the SUM-SG and MAX-SG 
by pessimistic players with local information. 
We first consider general settings where 
the initial graph is not a tree and 
multiple players perform edge swaps simultaneously. 
We show that the two games admit best response cycles. 
We then demonstrate that when the initial graph is a tree, 
the SUM-SG and MAX-SG have 
the FIP and converges to an equilibrium 
within polynomial number of edge swaps. 

\subsection{Impossibility in general settings} 

We first present several necessary conditions 
for an evolution of the SUM-SG and MAX-SG 
by players with local information 
to reach an equilibrium. 
We first present the necessary visibility for each player 
to change their strategies. 

\begin{theorem}
\label{t:UHP-SUMMAX-P-1-2}
In the SUM-SG and MAX-SG, 
when $k\leq2$, 
no player is unhappy in an arbitrary graph. 
Thus, any graph is a swap equilibrium. 
\end{theorem}
\begin{proof}
When $k=1$, no player can perform an edge swap 
because $V_{G,1}(u) \setminus N_{G_t}(u) = \emptyset$ at 
any $u \in V$. 

When $k=2$, we first consider the SUM-SG. 
Assume player $u$ is unhappy because of edge swap $(v,w)$ in graph $G$. 
Let $G'$ be the graph obtained by this edge swap. 
Thus, $d_{G}(u,w)=2$ and $d_{G}(u,w) - d_{G'}(u,w) = 1$. 
In a worst-case global graph, 
$w$ has no adjacent vertex other than those in $V_{G,2}(u)$ and 
the cost of $u$ decreases by at most one by this edge swap. 
In $G'$, $v$ must be reachable from $u$. 
There exists at least one player that is in $V_{G,2}(u)$ 
and adjacent to $v$, otherwise 
$v$ is not reachable from $u$ in a worst-case global graph. 
Hence, $d_{G}(u,v) - d_{G'}(u,v) = -1$. 
Additionally, $d_{G}(u,x) - d_{G'}(u,x) \leq 0 $ for any 
$x \in V_{G,2}(u)\setminus \{v,w\}$. 
Consequently, $\Delta_u(v,w) \leq 0$ and  
$u$ is not unhappy in $G$. 

Next, we consider the MAX-SG. 
Assume player $u$ is unhappy because of edge swap $(v,w)$ in graph $G$. 
Thus, $d_{G}(u,w)=2$ and $w$ is the only player at distance $2$ from $u$ 
in $V_{G,2}(u)$ 
otherwise $u$ is not unhappy  because of the edge swap $(v,w)$ in $G$. 
Let $G'$ be the graph obtained by this edge swap. 
In a worst-case global graph, 
$w$ has no adjacent vertex other than those in $V_{G,2}(u)$ and 
the cost of $u$ is expected to be reduced to $1$. 
By the same discussion above, 
$v$ is reachable from $u$ in $G'$, 
however $d_{G'}(u,v) = 2$. 
Hence, the maximum distance from $u$ to players in $V_{G',2}(u)$ is still two, 
thus $\Delta_u(v,w) \leq 0$. 
Hence, $u$ is not unhappy in $G$. 
\end{proof}

The following theorem justifies our assumption of a single edge 
swap in each transition. 
\begin{theorem}
\label{t:CYCLE-MULTIPLE}  
When $k\geq 3$, 
if multiple players change their strategies simultaneously, 
the SUM-SG and MAX-SG admit best response cycles. 
\end{theorem}
\begin{proof}
We first consider the SUM-SG. 
Consider a path of four players $u$, $v$, $w$, and $x$ 
aligned in this order. 
When $k \geq 3$, 
the two endpoint players $u$ and $x$ are unhappy because of 
the edge swap $(v,w)$ and $(w,v)$, respectively. 
If the two players perform the edge swaps simultaneously, 
the resulting graph is again a path graph, 
where $u$ and $x$ are unhappy. 

The above example is also a best response cycle in the MAX-SG. 
\end{proof}

Finally, we consider dynamics of SGs 
starting from an arbitrary initial graph. 
Lenzner presented a best response cycle for the SUM-SG 
by players with global information~\cite{L11}. 
During the evolution, the distance to any player 
from a moving player is always less than four 
and we can apply the result to 
the SUM-SG by pessimistic players 
with $k$-local information for $k \geq 3$.
In addition, the edge swaps are also best responses in the MAX-SG. 
Hence, we can also apply the result to 
the MAX-SG by pessimistic players 
with $k$-local information for $k\geq 3$. 
We have the following theorem. 
\begin{theorem}
\label{t:CYCLE-SUMMAX-P}
When $k\geq 3$, 
there exists an initial graph from which the SUM-SG and MAX-SG admit 
a best response cycle. 
\end{theorem}

In the following, we concentrate on 
the SUM-SG and MAX-SG by pessimistic players 
with $k$-local information for $k\geq 3$ starting from a tree. 
As defined in the preliminary, 
a single player changes its strategy in each transition.

\subsection{Convergence from an initial tree}

In this section, we show that the SUM-SG and MAX-SG have the FIP. 
For players with global information, 
generalized ordinal potential functions for the SUM-SG~\cite{L11} 
and MAX-SG~\cite{KL13} have been proposed. 
We can use these generalized ordinal potential functions 
for pessimistic players with local information. 

\begin{theorem}
\label{t:FIP-SUM}
If $G_0$ is a tree, any SUM-SG $(G_0,k)$ has the FIP and 
reaches a sum-swap equilibrium within $\Order(n^3)$ edge swaps. 
\end{theorem}
\begin{proof}
We show that $\Phi_{\text{SUM}} = SC(G)$ is 
a generalized ordinal potential function for the SUM-SG 
irrespective of the value of $k$. 
Consider a tree $G_t$ where an arbitrary unhappy player $u$ 
performs an edge swap $(v,w)$ that yields a new graph $G_{t+1}$. 
We have $\Delta_u(v,w) > 0$. 

Lenzner showed that for players with global information 
$SC(G_t)-SC(G_{t+1}) \geq 2$ holds 
if $c_u(G_t)-c_u(G_{t+1}) > 0$~\cite{L11}. 
Since $\Delta_u(v,w)$ considers the worst case graph, 
$c_u(G_t)-c_u(G_{t+1}) \geq \Delta_u(v,w) > 0$ holds. 
Consequently, $\Phi_{\text{SUM}}$ is a generalized ordinal potential function 
for the SUM-SG. 

Lenzner also showed that 
when the graph is a path of $n$ vertices, 
$\Phi_{\text{SUM}}$ achieves the maximum value of $\Theta(n^3)$, and 
if the graph is a star of $n$ vertices, 
$\Phi_{\text{SUM}}$ achieves the minimum  value of $\Theta(n^2)$. 
Hence, the number of edge swaps is $\Order(n^3)$. 
\end{proof}

We next show the FIP of the MAX-SG. 
Kawald and Lenzner presented a generalized ordinal function 
for the MAX-SG by players with global information~\cite{KL13}. 
Their generalized ordinal function is 
an $n$-tuple of players' costs, where 
the players are sorted in the descending order of their costs.  
We apply their function to the MAX-SG by pessimistic players 
with local information, 
however, we found that their proof needs small correction. 
In the following, we provide a new proof for their function. 

Consider the case where an unhappy player $u$ performs an 
edge swap $(v,w)$ in $G_t$ and a new graph $G_{t+1}$ is formed. 
Graph $G_t\setminus \{\{u,v\}\}$ consists of two trees and 
let $G_t^u$ be the tree containing vertex $u$ and 
$G_t^v$ be the tree containing vertex $v$. 
We have the following two lemmas. 

\begin{lemma}\cite{KL13} 
\label{l:COST-MAX-P-a}
Any player $x \in V(G_t^u)$ satisfies 
$c_x(G_t) > c_x(G_{t+1})$. 
\end{lemma}

\begin{lemma}
\label{l:COST-MAX-P-b}
Any player $y \in V(G_t^v)$ satisfies at least one of the following 
two conditions; 
(i) there exists a player $x \in V(G_t^u)$ that satisfies 
$c_x(G_t) > c_y(G_{t+1})$ and 
(ii) $c_y(G_t) \geq c_y(G_{t+1})$. 
\end{lemma}
\begin{proof}
For an arbitrary player $y \in V(G_t^v)$, let 
\begin{eqnarray*}
x &=& \argmax_{x' \in V(G_t^u)} \{d_{G_t^u}(u,x')\}, \\ 
z &=& \argmax_{z' \in V(G_t^v)} \{d_{G_t^v}(v,z')\}, \ \text{and}\\
z' &=& \argmax_{z'' \in V(G_t^v)} \{d_{G_t^v}(y,z'')\}. 
\end{eqnarray*}
We have the following three equations; 
\begin{eqnarray*}
c_x(G_t) &=& d_{G_t}(x,u) + 1 + d_{G_t}(y,z) = d_{G_t}(x,z) \\ 
c_y(G_t)&=& \max \{d_{G_t}(x,y), d_{G_t}(y, z')\} \\ 
c_y(G_{t+1}) &=& \max \{d_{G_{t+1}}(x,y), 
d_{G_{t+1}}(y, z')\}. 
\end{eqnarray*} 
If $d_{G_{t+1}}(x,y)>d_{G_{t+1}}(y,z')$, we have 
\begin{eqnarray*}
c_x(G_t)-c_y(G_{t+1}) 
&=& d_{G_t}(x,z) - d_{G_{t+1}}(x,y) \\ 
&=& d_{G_t}(x,u) + 1 + d_{G_t}(v,z) \\ 
&& - (d_{G_{t+1}}(x,u) + 1 + d_{G_{t+1}}(w,y)) \\ 
&=& d_{G_t}(x,u) + 1 + d_{G_t}(v,z) \\ 
&& - (d_{G_t}(x,u) + 1 + d_{G_t}(w,y)) \\ 
&=& d_{G_t}(v,z) - d_{G_t}(w,y) > 0. 
\end{eqnarray*}
The third equation holds because during the 
transition from $G_t$ to $G_{t+1}$ 
any distance between vertices in $V(G_t^u)$ 
(and $V(G_t^v)$, respectively) does not change. 
The last line is bounded by $0$ because 
\begin{eqnarray*}
d_{G_{t+1}}(u,y) 
&=& 1 + d_{G_{t+1}}(w,y) \\ 
&=& 1 + d_{G_t}(w,y) \\ 
&<& 1 + d_{G_t}(v,z),   
\end{eqnarray*}
otherwise $u$ was not unhappy in $G_t$. 
Thus, the first condition is satisfied. 

If $d_{G_{t+1}}(x,y) \leq d_{G_{t+1}}(y,z')$, we have 
\begin{eqnarray*}
c_y(G_{t+1}) 
&=& d_{G_{t+1}}(y,z') \\ 
&=& d_{G_{t}}(y,z') \\  
&\le& \max{\{ d_{G_{t}}(x,y), d_{G_{t}}(y,z') \}} \\ 
&=& c_y(G_t) 
\end{eqnarray*}
and the second condition is satisfied. 
\end{proof}

We define $\Phi_{\text{MAX}}(G)$ for a graph $G$ as an $n$-tuple 
$(c_{u_1}(G), c_{u_2}(G), \ldots, c_{u_n}(G))$ where 
$c_{u_i}(G) \geq c_{u_{i+1}}(G)$ 
for $i=1,2, \cdots, n-1$. 
We assume that ties are broken arbitrarily. 
We then consider lexicographic ordering 
of $n$-tuples. 
For two $n$-tuples $C=(c_1, c_2, \ldots, c_n)$ 
and $C'=(c'_1, c'_2, \ldots, c'_n)$ 
where $c_i, c'_i \in {\mathbb Z}$ for $i=1,2,\cdots,n$, 
when $\Delta=(c_1-c'_1, c_2-c'_2, \ldots, c_n-c'_n) \neq \bm{0}$ 
and the leftmost non-zero entry of $\Delta$ is positive, 
we say $C$ is lexicographically larger than $C'$, 
denoted by $C >_{lex} C'$. 

\begin{theorem}
\label{t:FIP-MAX} 
If $G_0$ is a tree, a MAX-SG $(G_0,k)$ has the FIP and 
reaches a max-swap equilibrium within $\Order(n^3)$ edge swaps. 
\end{theorem}
\begin{proof} 
We demonstrate that any transition from $G_t$ to $G_{t+1}$ 
satisfies 
$\Phi_{\text{MAX}}(G_t) >_{lex} \Phi_{\text{MAX}}(G_t+1)$. 
Let $u$ be the moving player in $G_t$ 
and $x$ be the player with the maximum cost in $G_t^u$. 
Then, in $G_t^v$, 
there may exist a player with larger cost than $c_x(G_t)$. 
We sort these players in the descending order of their costs 
and let $y_1, y_2, \ldots, y_{p}$ be the obtained sequence of 
players and $y_{p+1}, \ldots, y_q$ be the remaining players in $G_t^v$. 

We first show that 
any $y_j$ ($1 \leq j \leq p$) satisfies 
$c_{y_j}(G_t) \geq c_{y_j}(G_{t+1})$. 
If the second condition of Lemma~\ref{l:COST-MAX-P-b}
holds for all $y_1, y_2, \ldots, y_p$, 
we have the property. 
Otherwise, 
there exists $y_{j}$ that does not satisfy the second condition 
but the first condition. 
However, we have $c_{y_j}(G_t) \geq c_{x}(G_t)$ 
and by the proof of Lemma~\ref{l:COST-MAX-P-b}, 
$c_x(G_t) > c_{y_j}(G_{t+1})$ holds. 
This is a contradiction and all $y_j$ satisfies 
$c_{y_j}(G_t) \geq c_{y_j}(G_{t+1})$. 

Then we consider a player 
$v \in V(G_t^u) \cup \{y_{p+1}, y_{p+2}, \ldots, y_{q}\}$. 
By Lemma~\ref{l:COST-MAX-P-a} and Lemma~\ref{l:COST-MAX-P-b}
such player $v$ satisfies $c_v(G_{t+1}) < c_x(G_t)$. 
Consequently, we have 
$|\{v \mid c_v(G_t) \geq c_x(G_t)\}| > 
|\{v \mid c_v(G_{t+1}) \geq c_x(G_t)\}|$, and 
$\Phi_{\text{MAX}}(G_i) >_{lex} \Phi_{\text{MAX}}(G_{t+1})$. 

We can bound the number of edge swaps in the same manner as 
\cite{KL13}. 
\end{proof}

By Theorem~\ref{t:FIP-SUM} and \ref{t:FIP-MAX}, 
when an initial graph is a tree, 
the SUM-SG and MAX-SG by pessimistic players with local information 
converge to a sum-swap equilibrium and max-swap equilibrium, 
respectively 
within $\Order(n^3)$ edge swaps.

\section{PoA for pessimistic players} 

In this section, we analyze PoA of the SUM-SG and MAX-SG 
by pessimistic players with local information. 
Alon et al. showed that for players with global information, 
the diameter of a tree swap equilibrium 
is constant for the two cost functions, thus 
PoA is also constant~\cite{ADHL13}. 
On the other hand, 
our results show the clear contrast by the value of $k$.  
When $k=3$, there exists a sum-swap equilibrium 
of diameter $\Theta(n)$ and 
a max-swap equilibrium of diameter $\Theta(n)$. 
Thus, PoA is $\Theta(n)$ for both games. 
When $k\geq4$, the diameter of any sum-swap equilibrium 
is at most two and 
that of any max-swap equilibrium is at most three. 
Thus, the PoA is bounded by a constant for both games. 

In the following, we consider a path in a graph. 
A path $P$ of length $\ell$ is denoted by 
a sequence of vertices on it, i.e., 
$P=v_0 v_1 \ldots v_{\ell}$. 
The set of vertices that appear on $P$ is denoted by $V(P)$ and 
the set of edges of $P$ is denoted by $E(P)$. 
Given a tree $G=(V,E)$ and a path 
$P=v_0 v_1 v_2 \ldots v_{\ell}$ in $G$, 
consider the forest $G'=(V, E \setminus E(P))$ and 
let $T_{G,P}(v_i)$ denote the connected component (thus, a tree) 
containing $v_i$. 
We consider $v_i$ as the root of $T_{G,P}(v_i)$ 
when we address the depth of $T_{G,P}(v_i)$. 
The following lemma provides a basic technique to check the existence of 
an unhappy player. 
\begin{lemma}
\label{l:PATH-SUM-P-3}
In the SUM-SG, 
when $k=3$, a player $u$ in a tree $G$ is unhappy 
if and only if 
there exists a path $P=uvw$ that satisfies the following two 
conditions; 
(i) the depth of $T_{G,P}(v)$ is at most one, and 
(ii) $|V(T_{G,P}(v))| < |N_{T_{G,P}(w)}(w)|$. 
\end{lemma}

\begin{proof}
We first show that $u$ is unhappy if the two conditions hold. 
Assume that there is a path $P=uvw$ 
satisfying the two conditions. 
See Figure~\ref{fig:3trees}. 
Let $G'$ be the graph obtained by the edge swap $(v,w)$ at $u$ in $G$. 
For every $x \in V(T_{G,P}(v))$, 
$d_{G'}(u,x)=d_{G}(u,x)+1$ and 
for every $y \in N_{T_{G,P}(w)}(w)$ 
$d_{G'}(u,y)=d_{G}(u,y)-1$. 
By condition (i), $u$ knows that the edge swap $(v,w)$ 
increases the distance to $x \in V(T_{G,P}(v))$. 
By condition (ii), $u$ knows that the edge swap $(v,w)$ decreases 
its cost by at least $|N_{T_{G,P}(w)}(w)|$. 
In the worst-case global graph, 
$w$ has no adjacent players other than $N_{T_{G,P}(w)}(w)$. 
Hence, 
\begin{eqnarray*}
\Delta_{u}(v,w)
&=& |N_{T_{G,P}(w)}(w)| - |V(T_{G,P}(v))| > 0, 
\end{eqnarray*}
and $u$ is unhappy because of this edge swap $(v,w)$.  

\begin{figure}[t]
\centering 
\includegraphics[width=5cm]{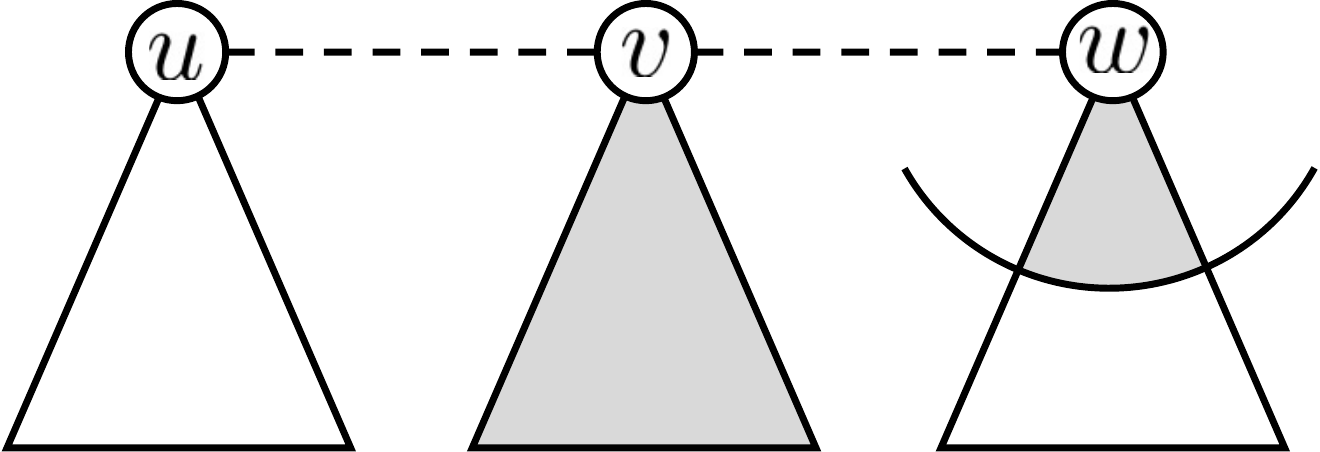}
\caption{$T_{G,P}(u)$, $T_{G,P}(v)$, and $T_{G,P}(w)$.}
\label{fig:3trees}
\end{figure}

Next, we show that $u$ is unhappy in $G$ 
only if the two conditions hold. 
Consider the case where for any path $P=uvw$,  
(i') the depth of $T_{G,P}(v)$ is larger than one, or 
(ii') $|V(T_{G,P}(v))| \geq |N_{T_{G,P}(w)}(w)|$ holds. 
We show that any player $u \in V$ is not unhappy. 
We check an arbitrary edge swap $(v',x')$ at $u$. 
Thus, $v' \in N_G(u)$ and $x' \in V_{G,3}(u) \setminus N_G(u)$. 
$G$ must have a path between $v'$ and $x'$, 
otherwise the edge swap $(v',x')$ disconnects the players. 
If $u$ cannot see this path, in the worst case global graph, 
$v'$ is not reachable from $u$. 
Hence, $G$ contains a path $u v' w' x'$ or $u v' x'$. 

If $G$ contains a path $u v' w' x'$, 
the edge swap $(v', x')$ satisfies 
$d_G(u,x') - d_{G'}(u,x')=2$, $d_G(u,v') - d_{G'}(u,v')=-2$, and 
$d_G(u,w') - d_{G'}(u,w')=0$. 
The worst-case global graph for the edge swap $(v', x')$ is 
a graph where $x'$ is not adjacent to any other vertex in 
$V(G) \setminus V_{G,3}(u)$. 
Thus, $\Delta_{c_u}(v',x') \leq 0$. 
Hence, $u$ is not unhappy with respect to the edge swap $(v',x')$. 

If $G$ contains a path $P'=u v' x'$ and condition (i') holds, 
there exist vertices $v'_1, v'_2 \in T_{G,P'}(v')$ that form 
a path $u v' v'_1 v'_2$. 
In the worst case global graph for the edge swap $(v', x')$, 
$v'_2$ has many children whose distance from $u$ 
increases by one in $G'$. 
Hence, $\Delta_{c_u}(v',x') \leq 0$ and $u$ is not unhappy 
with respect to the edge swap $(v', x')$. 

If $G$ contains a path $P' = u v' x'$ and condition (ii') holds, 
in the worst-case global graph for the edge swap $(v', x')$, 
the number of players whose distance from $u$ decreases by one 
with the edge swap $(v', x')$ is $|N_{T_{G,P}(w)}(w)|$ 
and the number of players whose distance from $u$ increases 
by one is lower bounded by $|V(T_{G,P}(v))|$. 
Thus, 
$\Delta_{c_u}(v',x') \geq |N_{T_{G,P}(w)}(w)| - |V(T_{G,P}(v))| 
\leq 0$ holds and 
player $u$ is not unhappy with respect to the edge swap $(v', x')$. 

Consequently, $u$ is not unhappy with respect to any edge swap. 
\end{proof}

\begin{theorem} 
\label{t:POA-SUM-P-3}
When $n \geq 13$ and $k \leq 3$, 
$\PoA_{\text{SUM}}(n,n-1,k) = \Theta(n)$. 
\end{theorem}
\begin{proof}
We present a sum-swap equilibrium of diameter $\Theta(n)$. 
We define a tree $TS(p)$ with a spine path of length $p$ 
as follows: 
For $i=1,2,\cdots, p$, 
$H_i$ is a tree, where $a_i$ has four children 
$b_i$, $c_i$, $d_i$, and $e_i$. 
For $i=0,p+1$, $H_i$ is a tree rooted at $a_i$ with 
three children $b_i$, c$_i$, and $d_i$. 
$TS(p)$ is a tree defined by 
\begin{eqnarray*}
V(TS(p)) &=& \bigcup_{i=0}^{p+1} V(H_i) \\ 
E(TS(p)) &=& \bigcup_{i=0}^{p+1} E(H_i) \cup 
\{\{a_0, e_1\}, \{e_{p}, a_{p+1}\}\} \cup 
\bigcup_{i=1}^{p-1} \{\{e_i,e_{i+1}\}\}. 
\end{eqnarray*}
Figure~\ref{fig:equ-sumsg-3} shows $TS(7)$ as an example. 
\begin{figure}[t]
\centering 
\includegraphics[width=6cm]{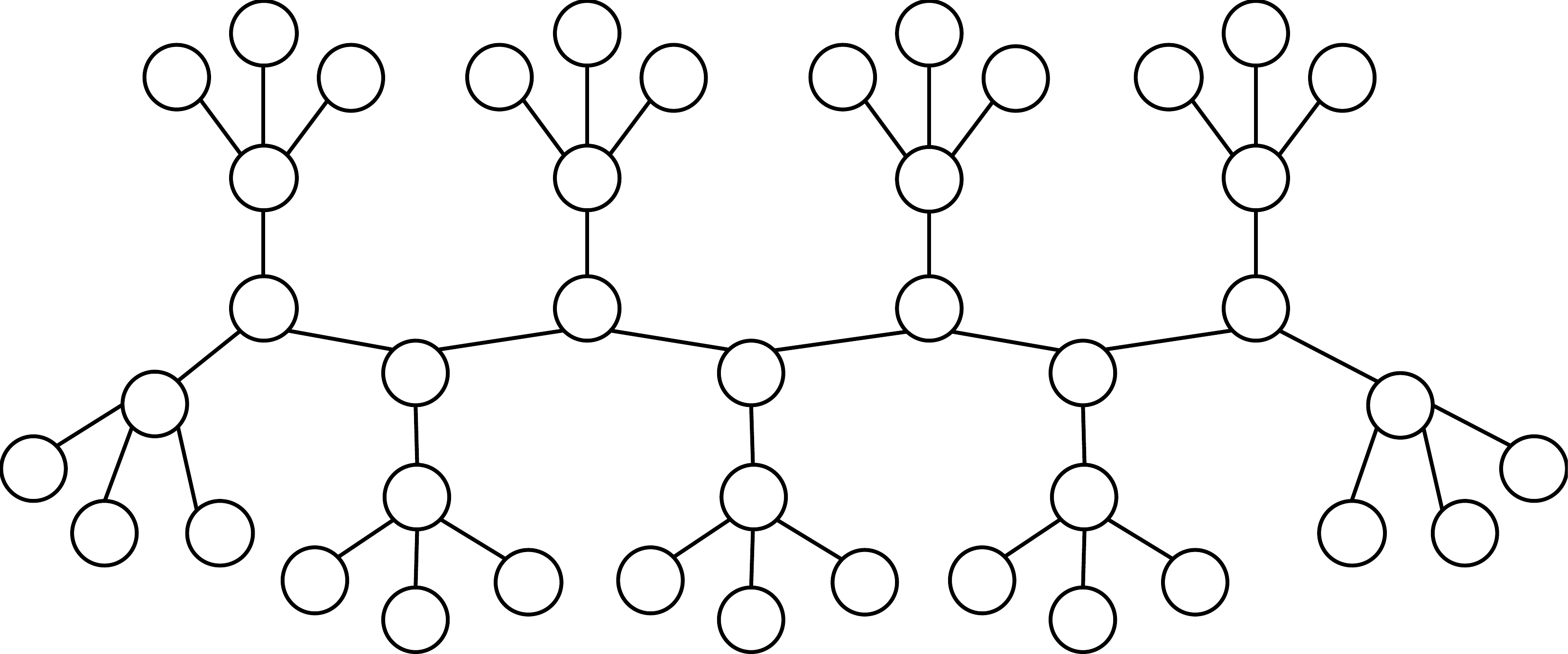}
\caption{$TS(7)$}
\label{fig:equ-sumsg-3}
\end{figure}

We show that for each $u \in V(TS(p))$ ($p \geq 3$), 
any path $P=uvw$ does not satisfy the two conditions of 
Lemma~\ref{l:PATH-SUM-P-3}. 
First, consider the case where $u$ is a leaf of $TS(p)$. 
Then, $v$ is $a_i$ for some $i=0,1,\cdots, p+1$. 
If $w$ is a leaf, the depth of $T_{G,P}(v)$ is larger than one, and 
the first condition of Lemma~\ref{l:PATH-SUM-P-3} is not satisfied. 
If $w$ is an internal vertex, 
$|V(T_{G,P}(v))|=3$ and $|N_{T_{G,P}(w)}(w)| = 3$, and 
the second condition of Lemma~\ref{l:PATH-SUM-P-3} 
is not satisfied. 

Second, consider the case where $u$ is an internal vertex of $TS(p)$. 
Then, $v$ is also an internal vertex 
otherwise we cannot find $w$. 
When the depth of $T_{G,P}(v)$ is larger than one, 
the first condition of Lemma~\ref{l:PATH-SUM-P-3} 
is not satisfied. 
When the depth of $T_{G,P}(v)$ is one, 
$v$ is $a_i$ for some $i=0,1,\cdots, p+1$ and 
$w$ is a leaf. 
Thus, $|V(T_{G,P}(v))|=3$ and $|N_{T_{G,P}(w)}(w)| = 1$, and 
the second condition of Lemma~\ref{l:PATH-SUM-P-3}  
is not satisfied. 

Thus, every player is not unhappy and 
$TS(p)$ is a sum swap equilibrium. 

We then calculate the social cost in $TS(p)$, 
that consists of $5p+8$ vertices. 
\begin{eqnarray*}
&& SC(TS(p)) \\ 
&=& 
\sum_{u \in V(TS(p))} \sum_{v \in V(TS(p))} d_{TS(p)}(u,v) \\ 
&=& 
\sum_{i=0}^{p+1} \sum_{u \in V(H_i)} \sum_{v \in V(TS(p))} d_{TS(p)}(u,v) \\ 
&=& 
2 \sum_{u \in V(H_0) \cup V(H_1)} \sum_{v \in V(TS(p))} d_{TS(p)}(u,v) \\ 
&& + \sum_{i=2}^{p-1} \sum_{u \in V(H_i)} \sum _{v \in V(TS(p))} d_{TS(p)}(u,v) \\ 
&=& (45p^2 + 293p + 252) \\ 
&& + (\frac{25}{3}p^3 + 65p^2 - \frac{286}{3}p-136) \\ 
&=& \Theta(p^3) \\  
&=& \Theta(n^3) 
\end{eqnarray*}

When $n \neq 5p+8$ for any integer $p$, 
we have the same bound by attaching extra vertices to some $a_i$. 

Since star is a sum-swap equilibrium with the minimum social cost, 
the PoA of $TS(p)$ is $\Theta(n)$. 
\end{proof}

By $TS(p)$, we have the following corollary. 
\begin{corollary}
\label{c:DIAM-SUM-P-3}
There exists a sum-swap equilibrium of 
diameter $\Theta(n)$ for any $n \geq 13$. 
\end{corollary}

We now demonstrate that when $k \geq 4$, 
sum-swap equilibrium for pessimistic players 
with $k$-local information achieves the same PoA as 
that with global information. 

We first show the following lemma. 
\begin{lemma}
\label{l:PATH-SUM-P-GEQ4}
In an arbitrary tree $G$ whose diameter is larger than two, 
there exists a path $P=vabw$ that satisfies the 
following two conditions; 
(i) $|V_{G_{T,P}(a),2}(a)| \leq |V_{G_{T,P}(b),2}(b)|$, 
and 
(ii) the depth of $T_{G,P}(a)$ is at most two. 
\end{lemma}
\begin{proof}
We prove the lemma by induction. 
There exists at least one path of length at least three in $G$. 
We choose a path $P=vabw$ arbitrarily. 
Let $d_a$ and $d_b$ be the depth of $T_{G,P}(a)$ and 
$T_{G,P}(b)$, respectively. 

\noindent{\bf (Base case.)~}  
Consider the case where $\max\{d_a, d_b\} \leq 2$. 
Let $P_{rev} = wbav$. 
The two paths $P$ and $P_{rev}$ satisfies the second condition 
and either $P$ or $P_{rev}$ satisfies the first condition. 
Thus, the statement holds when $\max\{d_a, d_b\} \leq 2$. 

\noindent{\bf (Induction step.)~}  
Assume the statement holds when $\max\{d_a,d_b\} \leq d-1$ 
for $d \geq 3$. 
Consider the case where  $\max\{d_a,d_b\} = d$. 
Without loss of generality, we assume $d_a=d \geq 3$. 
Hence, there exists at least one path 
$P'=a a' b' w'$ in $T_{G,P}(a)$. 
See Figure~\ref{fig:path-SUM-SG-4}.
Let $d_{a'}$ and $d_{b'}$ be the depth of $T_{G,P'}(a')$ and 
$T_{G,P'}(b')$, respectively. 
Clearly, $d_{a}>d_{a'}$ and $d_{a'}>d_{b'}$ hold and 
the statement holds by the induction hypothesis. 

\begin{figure}[t]
\centering 
\includegraphics[width=8cm]{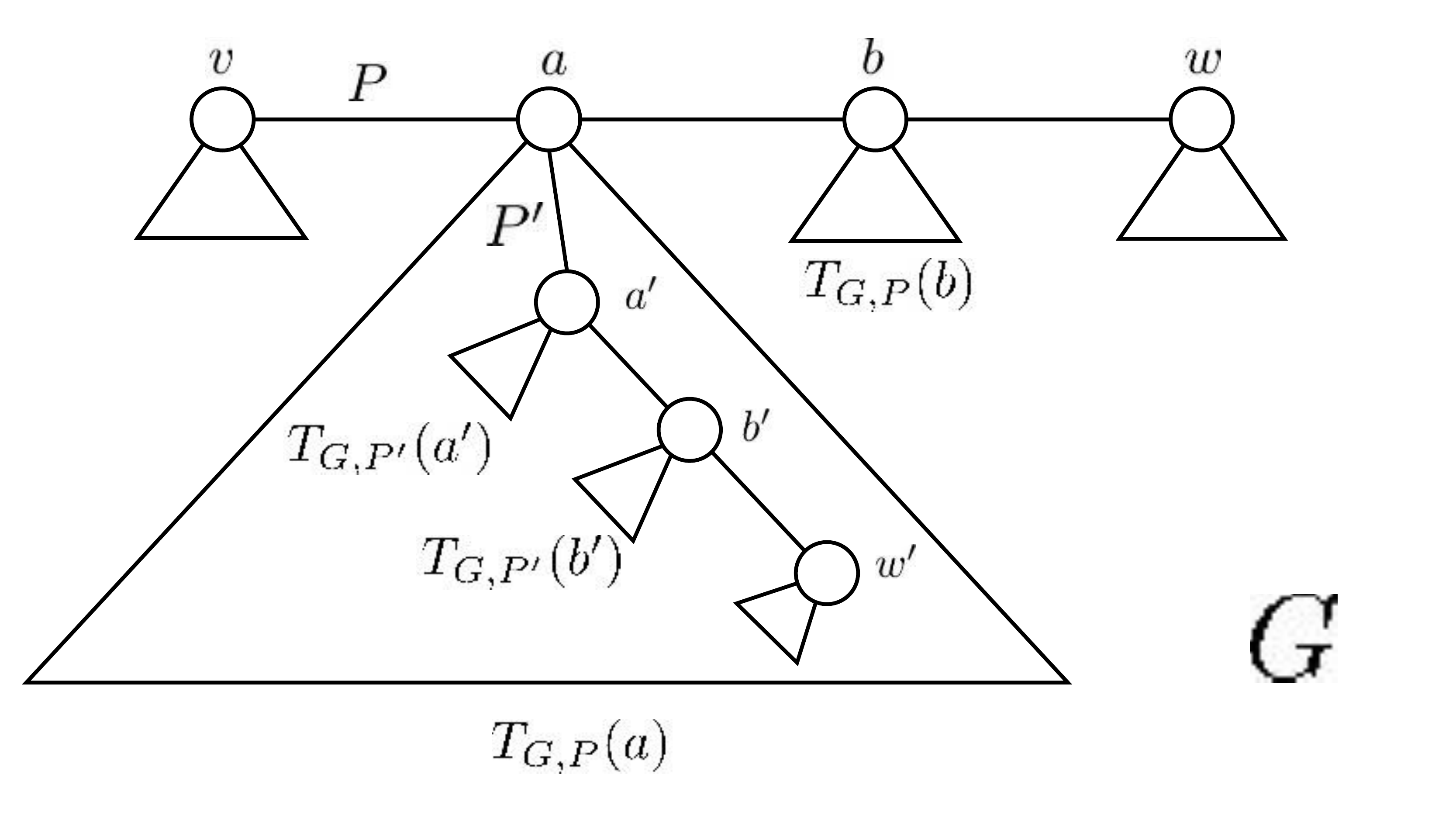}
\caption{Induction step of Lemma~\ref{l:PATH-SUM-P-GEQ4}.}
\label{fig:path-SUM-SG-4}
\end{figure}
\end{proof}

By Lemma~\ref{l:PATH-SUM-P-GEQ4}, 
in any graph $G$ whose diameter is larger than two 
there exists an unhappy player. 
\begin{theorem}
\label{t:POA-SUM-P-GEQ4}
When $G_0$ is a tree and $k \geq 4$, 
any sum-swap equilibrium is a star and 
$\PoA_{\text{SUM}}(n,n-1,k) = 1$. 
\end{theorem}
\begin{proof}
Assume that there exists a sum-swap equilibrium $G$ 
whose diameter is larger than two. 
By Lemma~\ref{l:PATH-SUM-P-GEQ4}, 
there exists a path $P=vabw$ that satisfies the two conditions. 
Hence, $v$ is unhappy because 
\begin{eqnarray*}
 \Delta c_v(a,b) \geq |V_{T_{G,P}(b),2}(b)| +1 - |V_{T_{G,P}(a),2}(a)| > 0. 
\end{eqnarray*}
This is a contradiction and $G$ is not a sum-swap equilibrium. 
Hence, the diameter of a sum-swap equilibrium is smaller than 
or equal to two and we have the statement. 
\end{proof}

Consequently, the ``visibility'' of pessimistic players 
has a significant effect on the PoA of the SUM-SG. 
We then demonstrate that this is also the case for the MAX-SG. 

\begin{theorem}
\label{t:POA-NMAX-P-3} 
When $n \geq 6$ and $k = 3$, 
$\PoA_{\text{MAX}}(n,n-1,k) = \Theta(n)$. 
\end{theorem}
\begin{proof}
We show that a tree shown in Figure~\ref{fig:equ-maxsg-3} 
is a max-swap equilibrium. 
\begin{figure}[t]
\centering 
\includegraphics[width=6cm]{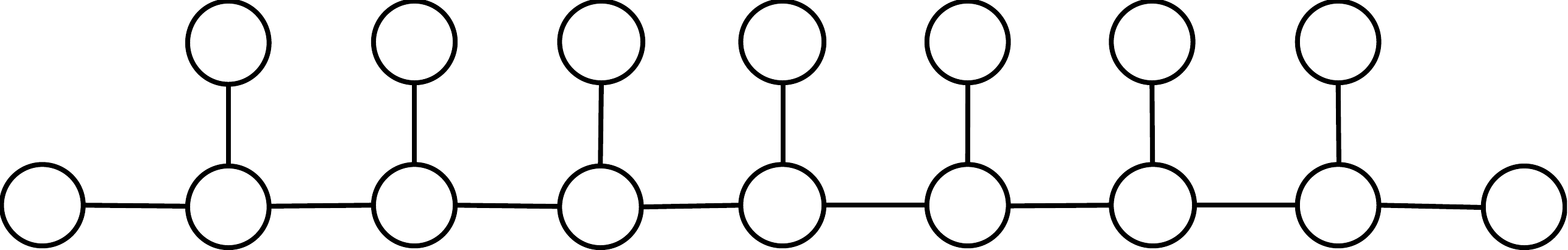}
\caption{A max-swap equilibrium of $16$ players.}
\label{fig:equ-maxsg-3}
\end{figure}

First, consider the two endpoint vertices. 
Each endpoint player has one incident edge and 
any edge swap involving this edge does not decrease the maximum distance 
to the vertices in its view. 
In the worst-case global graph the player at distance two 
has no other vertices. 
Thus, the two endpoint players are not unhappy. 

Second, consider other leaves. 
Each leaf player has one incident edge and 
any edge swap at a leaf increases the maximum distance 
to some vertex in its view. 
In the worst-case global graph, 
there is a long path starting from such a vertex. 
Hence, the leaf players are not unhappy. 

Finally, consider inner vertices. 
Each inner player has three edges but 
it cannot remove the edge connecting it to a leaf 
because such an edge swap disconnects 
the graph. 
If the player remove an edge incident to 
another inner vertex and create a new edge, 
by the same discussion as above, 
this edge swap increases the maximum distance 
to some vertex in its view and 
the player is not unhappy. 

Consequently, the graph shown in Figure~\ref{fig:equ-maxsg-3} 
is a max-swap equilibrium. 
By adding inner vertices (with its child), 
we have the similar equilibrium for any even $n\geq 6$. 
For odd $n\geq 6$, we attach an extra player to an inner vertex 
and obtain a max swap equilibrium. 

Since the star graph is an max-swap equilibrium with the minimum 
social cost, we have $\PoA = \Theta(n)$. 
\end{proof}

By the proof of Theorem~\ref{t:POA-NMAX-P-3}, 
we have following corollary.  
\begin{corollary}
\label{c:DIAM-MAX-P-3}
There exists a max-swap equilibrium of diameter $\Theta(n)$ 
for any $n \geq 6$. 
\end{corollary}

We now demonstrate that when $k \geq 4$, 
any MAX-SG by pessimistic players with $k$-local information 
achieves the same PoA as that of players with global information. 
The following lemma shows that 
in any tree of diameter larger than three, 
there is at least one unhappy player. 
\begin{lemma}
\label{l:PATH-MAX-P-GEQ-4}
In any tree $G$ whose diameter is larger than three, 
there exists a path $P=vabcw$ that satisfies the 
following two conditions;  
(i) $P$ starts from a leaf $v$, and 
(ii) the depth of $T_{G,P}(a)$ is at most one. 
\end{lemma}
\begin{proof}
There exists at least one path of length at least four in $G$. 
We arbitrarily choose a path $P=vabcw$ that starts 
from some leaf $v$. 
If the depth of $T_{G,P}(a)$ is smaller than two, 
the statement holds. 
If the depth of $T_{G,P}(a)$ is larger than one, 
choose a leaf $v'$ in $T_{G,P}(a)$ 
and its parent vertex, say $a'$. 
There exists at least one path $P'=v'a'b'c'w'$ 
and $P'$ satisfies the second condition. 
\end{proof}

\begin{theorem}
\label{t:POA-MAX-P-GEQ4}
When $G_0$ is a tree and $k \geq 4$, 
the diameter of any max-swap equilibrium is at most three and 
$\PoA_{\text{MAX}}(n,n-1,k) \leq 3/2$. 
\end{theorem}
\begin{proof}
Assume that there exists a max-swap equilibrium $G$ 
whose diameter is larger than three. 
By Lemma~\ref{l:PATH-MAX-P-GEQ-4}, 
there exists a path $P=vabcw$ such that 
$v$ is a leaf and the depth of $T_{G,P}(a)$ is at most one. 
Player $v$ is unhappy because $\Delta c_v(a,b) \geq 1$. 
This is a contradiction and $G$ is not a max-swap equilibrium. 
Thus, the diameter of any max-swap diameter is at most three. 

Because a equilibrium with the minimum cost is a star, 
the PoA is bounded by $3/2$. 
\end{proof}

\section{Swap games with non-pessimistic players} 

We demonstrated that when $k=2,3$, 
the PoA for pessimistic players is $\Theta(n)$ in 
the SUM-SG and MAX-SG. 
In this section, we introduce less pessimistic players 
to obtain smaller PoA for these cases. 
We consider two types of non-pessimistic players: 
A player $u$ is \emph{weakly pessimistic} 
if $u$ is unhappy when there exists 
an edge swap $(v,w)$ at $u$ 
such that $\Delta c_{u}(v,w) \geq 0$. 

A player $u$ is \emph{optimistic} 
if its $\Delta_{c_u}(v,w)$ is defined as 
\begin{eqnarray*}
 \Delta_{c_{u}}(v,w) = \max_{H \in {\mathcal G}_u} 
(c_{u}(H) - c_{u}(H')), 
\end{eqnarray*}
where $H'$ is a graph obtained by an edge swap 
$(v,w)$ at $u$ in $H \in {\mathcal G}$. 

Weakly pessimistic players and optimistic players 
do not perform any edge swap in the SUM-SG and MAX-SG 
when $k=1$. 
Different from Theorem~\ref{t:UHP-SUMMAX-P-1-2}, 
weakly pessimistic players change their strategies 
when $k=2$. 
However, when $k>2$, 
weakly pessimistic players cause a cycle of edge 
swaps from an initial path graph. 

\begin{example}
Let $P=u_0 u_1 u_2 \ldots$ be a path of $n$ ($\geq 2k$) 
weakly pessimistic players with $k$-local information. 
In the SUM-SG and MAX-SG, 
player $u_k$ is unhappy because of the edge swap $(u_{k-1},u_0)$. 
However after $u_k$ performs this edge swap, 
the graph is $u_{k-1} u_{k-2} \ldots u_0 u_k u_{k+1} \ldots u_n$ 
and $u_k$ is again unhappy because of the 
edge swap $(u_0,u_{k-1})$. 
By selecting $u_k$ forever, the graph never reach an 
equilibrium. 
\end{example} 

We now consider a more restricted \emph{round robin scheduling} 
for selecting moving players. 
In a round-robin scheduling, players have a fixed ordering and 
at each time step a moving player is selected according to this 
ordering. 
Consider $n$ players $u_1, u_2, \ldots, u_n$ and 
let the subscript $i$ indicate the order of player $u_i$. 
In $G_0$ if $u_1$ is unhappy, $u_1$ is selected as the moving player. 
Otherwise, we check $u_2, u_3, \ldots$ until we 
find an unhappy player. 
Thus, the unhappy player with the smallest order, say $j$, is 
selected as a moving player. 
In $G_1$ if $u_{j+1}$ is unhappy, $u_{j+1}$ is selected as 
the moving player. 
Otherwise, we check $u_{j+2}, u_{j+3}, \ldots$ 
until we find an unhappy player. 
If $u_n$ is selected in $G_t$, 
the check start with $u_1$ in $G_{t+1}$. 

However, the round-robin scheduling still admits 
a best response cycle. 
\begin{example}
Consider a a path $P_0 = u_3 u_1 u_4 u_2$ of four weakly pessimistic 
players with $k$-local information for $k \geq 2$. 
For the SUM-SG, 
in $P_0$, $u_1$ is unhappy because of the edge swap 
$(u_4, u_2)$. 
When $u_0$ performs this edge swap, 
a new path $P_1=u_3 u_1 u_2 u_4$ is formed. 
In $P_1$, $u_2$ is unhappy because of the edge swap 
$(u_1, u_3)$. 
When $u_2$ performs this edge swap, 
a new path $P_3=u_1 u_3 u_2 u_4$ is formed. 
In $P_3$, $u_3$ is unhappy because of the edge swap 
$(u_2, u_4)$. 
When $u_3$ performs this edge swap, 
a new path $P_4=u_1 u_3 u_4 u_2$ is formed. 
In $P_4$, $u_4$ is unhappy because of the edge swap 
$(u_3, u_1)$. 
When $u_4$ performs this edge swap, 
a new path $P_5=u_3 u_1 u_4 u_2 = P_0$ is formed. 
This is a best response cycle in the SUM-SG and 
MAX-SG. 
\end{example}

Moreover, we can show that there exists a best response 
cycle from an initial tree. 
\begin{theorem}
\label{t:CYCLE-SUMMAX-WP}
When $k \geq 3$, 
for weakly pessimistic players with $k$-local information, 
there exist infinitely many trees from which 
the SUM-SG and MAX-SG admit best response cycles 
under the round-robin scheduling. 
\end{theorem}
\begin{proof}
We present an initial tree $G_0$ for the SUM-SG 
by $n=2m+3$ weakly pessimistic 
players $u_1, u_2, \ldots, u_{2m+3}$ ($m=2,3,\ldots$). 
The players are divided into two subtrees rooted at 
$u_{2m+3}$ and $u_{2m+2}$, respectively; 
$u_{2m+3}$ has $m+1$ leaves $u_1, u_3, \ldots, u_{2m+1}$ and 
$u_{2m+2}$ has $m$ leaves $u_2, u_4,\ldots, u_{2m}$. 
Additionally, $G_0$ contains an edge connecting $u_{2m+3}$ 
and $u_{2m+2}$. 
See Figure~\ref{fig:seesaw} for $n=9$. 

\begin{figure}[t]
\centering 
\includegraphics[width=3cm]{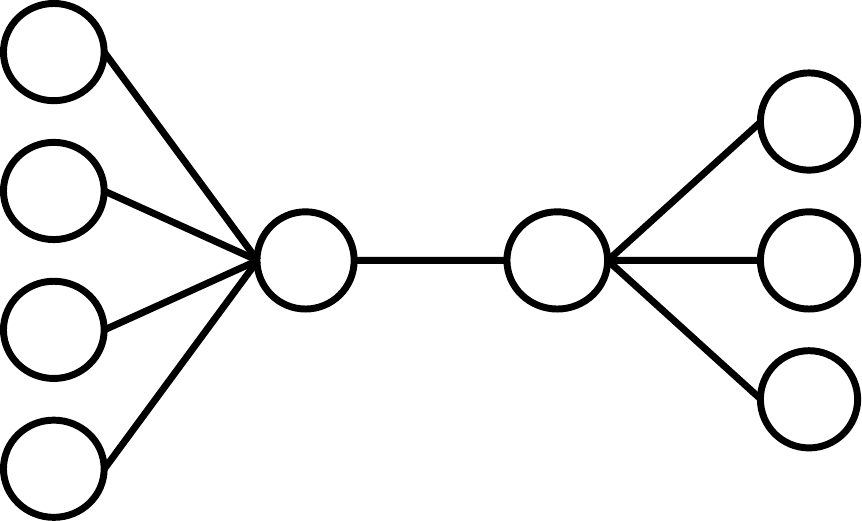}
\caption{An example of $G_0$ for nine players. 
The two inner vertices are $u_9$ with four leaves 
($u_1$, $u_3$, $u_5$, and $u_7$) 
and $u_8$ with three leaves ($u_2$, $u_4$, and $u_6$). }
\label{fig:seesaw}
\end{figure}

In $G_0$, $u_1$ is unhappy because of the edge swap 
$(u_{2m+3}, u_{2m+2})$; 
$u_{2m+3}$ and $u_{2m+2}$ 
have $m$ leaves except $u_1$. 
When $u_1$ performs this edge swap, 
a new graph $G_1$ is formed, where 
$u_{2m+3}$ has $m$ leaves and $u_{2m+2}$ has $m+1$ leaves. 
In $G_1$, $u_2$ is unhappy because of the edge swap 
$(u_{2m+2}, u_{2m+3})$. 
When $u_2$ performs this edge swap, 
a new graph $G_2$ is formed, where 
$u_{2m+3}$ has $m+1$ leaves and $u_{2m+2}$ has $m$ leaves. 
In this way, the leaves of a graph keep on changing their parent 
under the round robin scheduling. 
The seesaw game continues until $G_{2m}$ 
where $u_{2m+3}$ has $m$ leaves with even subscripts 
and $u_{2m+2}$ has $m+1$ leaves with odd subscripts. 
Player $u_{2m+2}$ is not unhappy in $G_{2m}$; 
the only possibility is an edge swap 
$(u_{2m+3}, u_{2\ell})$ 
for some $\ell \in \{2,4,\ldots, m\}$, 
but it increases its cost. 
Player $u_{2m+3}$ is not unhappy in $G_{2m}$ 
with the same reason. 
Player $u_1$ is unhappy in $G_{2m}$, 
and the leaf players start the seesaw game again. 

This cycle is also a best response cycle in the MAX-SG. 
\end{proof}

Next, we present the PoA for 
weakly pessimistic players with $2$-local information. 

\begin{lemma}
\label{l:PATH-SUM-WP-2}
In the SUM-SG and MAX-SG by weakly pessimistic players 
with $2$-local information, 
a player $u$ is unhappy if and only if 
there is a path $uvw$ where the degree of $v$ is two. 
\end{lemma}
\begin{proof}
We first consider the SUM-SG. 
If there is a path $uvw$ in $G$ such that the degree of $v$ is two, 
$u$ is unhappy because of the edge swap $(v,w)$. 
In the worst-case global graph, $w$ does not have any adjacent player 
except $v$ and the cost decreases by one. 
On the other hand, the edge swap $(v,w)$ increases the distance 
between $u$ and $v$ by one. 

If $u$ is unhappy because of the edge swap $(v,w)$ in $G$, 
$d_G(u,v)=1$ and $d_G(u,w) = 2$. 
Let $G'$ be the graph obtained by the edge swap $(v,w)$ at $u$ in $G$. 
Graph $G'$ must contain a path between $v$ and $w$ 
otherwise players are disconnected. 
Additionally, $v$ is adjacent to $w$ in $G$ 
otherwise in a worst case global graph 
$v$ is not reachable from $u$. 
If $v$ has other adjacent vertices than $u$ and $w$, 
the edge swap $(v,w)$ increases the cost of $u$ by at lease two 
while the edge swap decreases the cost of $u$ by at least one.
Hence, the degree of $v$ must be two in $G$. 
Consequently, we have the statement. 

Then we consider the MAX-SG. 
If there is a path $uvw$ where the degree of $v$ is two, 
$u$ is unhappy because of the edge swap $(v,w)$. 
If $u$ is unhappy because of the edge swap $(v,w)$, 
then the degree of $v$ is two in $G$, 
otherwise the edge swap increases the cost of $u$ by at least one. 
Consequently, we have the statement. 
\end{proof}

By Lemma~\ref{l:PATH-SUM-WP-2}, 
a graph is a swap equilibrium if it 
does not have any vertex of degree two. 
The graph shown in Figure~\ref{fig:equ-maxsg-3} 
is a swap equilibrium with diameter $\Theta(n)$. 
We have the following theorem. 
\begin{theorem}
\label{t:POA-SUMMAX-WP-2}
For weakly pessimistic players 
with $2$-local information, 
when $n \geq 4$, 
$\PoA_{\text{SUM}}(n,n-1,2) = \Theta(n)$ and 
$\PoA_{\text{MAX}}(n,n-1,2) = \Theta(n)$.  
\end{theorem}

We then present the PoA of the SUM-SG by 
weakly pessimistic players with $3$-local information. 

\begin{theorem}
\label{t:POA-SUM-WP-3}
For weakly pessimistic players with $3$-local information, 
when $n \geq 16$, 
$\PoA_{\text{SUM}}(n,n-1,3)=\Theta(n)$. 
\end{theorem}
\begin{proof}
We present a sum-swap equilibrium in the same manner as 
Theorem~\ref{t:POA-SUM-P-3}. 
For $i=1,2,\cdots, p$, 
$H'_i$ is a tree in which $a_i$ has five children 
$b_i$, $c_i$, $d_i$, $e_i$, and $f_i$. 
For $i=0,p+1$, $H'_i$ is a tree rooted at $a_i$ with 
four children $b_i$, c$_i$, $d_i$, and $e_i$. 
Then $TS'(p)$ is is a tree defined by 
\begin{eqnarray*}
V(TS(p)) &=& \bigcup_{i=0}^{p+1} V(H_i) \\ 
E(TS(p)) &=& \bigcup_{i=0}^{p+1} E(H_i) \cup
\{\{a_0, f_1\}, \{f_{p}, a_{p+1}\}\} \cup 
\bigcup_{i=1}^{p-1} \{\{f_i,f_{i+1}\}\}. 
\end{eqnarray*}
Figure~\ref{fig:EQU-SUM-WP-3} shows $TS'(7)$ as an example. 
\begin{figure}[t]
\centering 
\includegraphics[width=8cm]{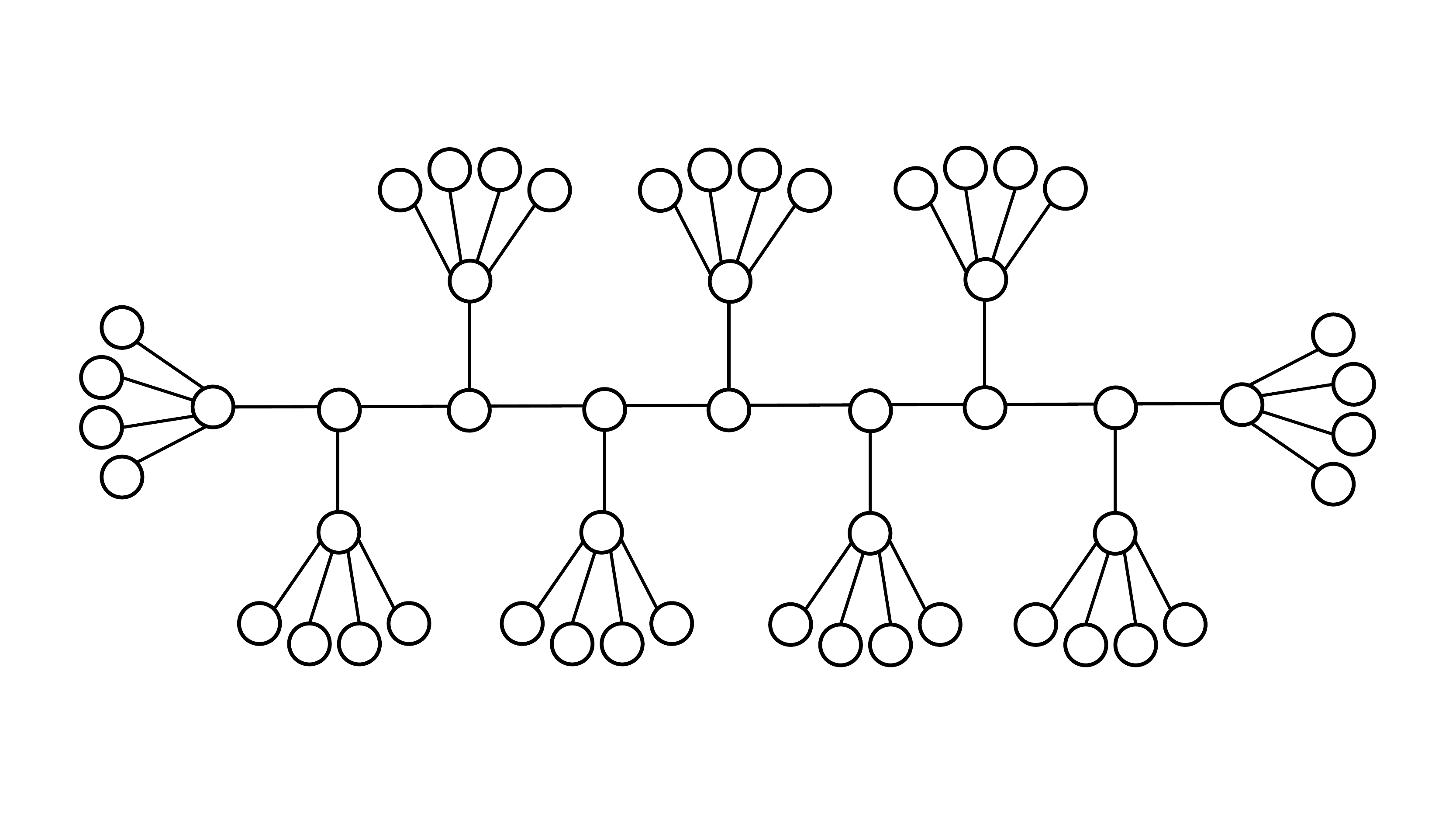}
\caption{$TS'(7)$}
\label{fig:EQU-SUM-WP-3}
\end{figure}

We omit the detailed proof 
because it easy to check whether every player is unhappy or not 
in $TS'(p)$. 
\end{proof}

By Theorem~\ref{t:POA-SUM-P-GEQ4}, 
we have the following theorem 
because when a pessimistic player $u$ is unhappy in graph $G$, 
$u$ is unhappy in $G$ when $u$ is a weakly pessimistic player. 
\begin{theorem} 
\label{t:POA-SUM-WP-GEQ4}
When $k \geq 4$, 
for weakly pessimistic players with $k$-local information, 
the diameter of a sum-swap equilibrium is at most two 
and $\PoA_{\text{SUM}}(n,n-1,k)=1$. 
\end{theorem}

On the other hand, 
the diameter of any max-swap equilibrium is smaller than three 
for weakly pessimistic players with $3$-local information. 
\begin{theorem}
\label{t:POA-MAX-WP-GEQ4}
When $k \geq 3$, 
for weakly pessimistic players with $k$-local information, 
the diameter of a max-swap equilibrium is at most two 
and $\PoA_{\text{MAX}}(n,n-1,k)=1$. 
\end{theorem}
\begin{proof}
Assume that there exists a max-swap equilibrium $G$ 
whose diameter is larger than two. 
Then, there exists a path $P = u v w x$ in $G$ that starts 
from a leaf $u$. 

If the depth of $T_{G,P}(v)$ is smaller than two, 
then $u$ is unhappy because of 
the edge swap $(v,w)$ that keeps the maximum distance to 
vertices in its view unchanged. 

If the depth of $T_{G,P}(v)$ is larger than one, 
then we can find a path $P' = u' v' w' x'$ such that 
$u'$ is a leaf and the depth of $T_{G,P'}(v')$ is at most one. 
More specifically, we choose an arbitrary leaf $u'$ 
with the largest depth in $T_{G,P}(v)$. 
There exists a path whose length is larger than three 
starting from $u'$ because there exists a path 
from $u'$ to $v$ and a path $v w x$. 
Then the depth of $T_{G,P'}(v')$ is at most one. 
If the depth of $T_{G,P'}(v')$ is zero, 
$u'$ is unhappy as mentioned above. 
If the depth of $T_{G,P'}(v')$ is one, 
$u'$ is unhappy because of the edge swap $(v', w')$ 
that keeps the maximum distance to 
any vertex in its view unchanged. 
Hence, $G$ is not a max-swap equilibrium. 

We now consider graph $G'$ whose diameter is two, 
i.e., $G'$ is a star. 
The unique inner vertex player, say $c$, is not unhappy 
because it cannot perform any edge swap since 
$V(G) \setminus N_G(c) \cup \{c\} = \emptyset$. 
Any leaf player is not unhappy 
because any edge swap changes its cost 
from two to three. 
Thus, an arbitrary star graph is a max-swap equilibrium 
and the minimum social cost of the MAX-SG is achieved 
by a star graph. 
Thus, $\PoA(n,n-1,k)=1$ for $k \geq 3$.  
\end{proof}

Finally, we consider optimistic players. 
An optimistic player $u$ with $k$-local information expects that 
a player at distance $k$ has a long path 
that $u$ cannot observe. 
Thus, $u$ always perform an edge swap to create an edge 
connecting itself to another player at distance $k$ if any. 

\begin{lemma}
\label{l:DIAM-SUMMAX-O-K}
For optimistic players with $k$-local 
information, 
the diameter of any swap equilibrium is smaller than $k$. 
\end{lemma}

Consequently, in a sum-swap equilibrium and a max-swap equilibrium, 
all optimistic players can observe the entire graph. 
We have the following theorem. 

\begin{theorem}
\label{t:POA-SUMMAX-O-K}
For optimistic players with $k$-local information, 
$\PoA_{\text{SUM}}(n,n-1,k)=1$  
and 
$\PoA_{\text{MAX}}(n,n-1,k)<3/2$.  
\end{theorem}

\section{Conclusion} 

In this paper, we introduced swap games with $k$-local information 
and investigated their dynamics and PoA. 
First, we showed that when $k \geq 4$, starting from a tree, 
the SUM-SG and MAX-SG by pessimistic players with 
$k$-local information promise convergence 
to an equilibrium with constant PoA. 
In other words, in a distributed environment, 
rational participants can construct a tree of small diameter 
without global information. 

We then introduced weakly pessimistic players 
to obtain a tree equilibrium with small PoA for $k\leq 3$. 
When $k=3$, the MAX-SG achieves a constant PoA, 
at the cost of best response cycles. 
Thus, relaxing pessimism does not promise distributed 
graph construction. 
Finally, we introduced optimistic players 
and presented the constant PoA of the SUM-SG and MAX-SG 
for any value of $k$. 

There are many interesting future directions. 
One is a better upper bound and a lower bound of the 
number of edge swaps during convergence. 
Based on \cite{KL13,L11}, we put our basis on the 
potential function for a game, 
however we do not know whether our upper bound is tight or not. 
The dynamics of non-pessimistic players is also an open problem. 

Although games with imperfect information have been 
investigated in game theory, 
to the best of our knowledge, 
there are few games where each player knows 
the existence of only a part of players. 
We hope games in this form 
open up new vistas for game theory and distributed computing.

\section*{Acknowledgment}

We would like to thank precious comments by 
Shuji Kijima in Kyushu University. 

\bibliographystyle{plain}
\bibliography{papers}

\end{document}